\documentclass[letterpaper, 10 pt, conference]{ieeeconf}  

\IEEEoverridecommandlockouts                              

\usepackage{graphicx}
\usepackage{mathtools}
\usepackage{amsmath} 
\usepackage{amssymb}  
\usepackage{hyperref}
\usepackage{mathrsfs}
\usepackage{bm}
\usepackage{bbm}
\usepackage{color}
\usepackage{colortbl}
\usepackage{dsfont}

\usepackage{algorithm}
\usepackage[noend]{algpseudocode}

\DeclareMathOperator*{\argmin}{argmin}
\newcommand{\norm}[1]{\left\lVert#1\right\rVert}

\usepackage[labelformat=simple]{subcaption}
\DeclareCaptionLabelSeparator{periodspace}{.\quad}
\usepackage{amsthm}
\theoremstyle{plain}

\newtheorem{prop}{Proposition}

\renewcommand{\baselinestretch}{0.970}

\title{\LARGE \bf
Free-Space Ellipsoid Graphs for Multi-Agent Target Monitoring
}

\author{Aaron Ray$^{1}$, Alyssa Pierson$^{2}$, Daniela Rus$^{1}$
\thanks{$^{1}$Computer Science and Artificial Intelligence Laboratory, Massachusetts Institute of Technology, Cambridge, MA 02139, USA
        {\tt\small \{aray, rus\}@csail.mit.edu}}%
\thanks{$^{2}$ Department of Mechanical Engineering, Boston University, Boston, MA 02215, USA
        {\tt\small pierson@bu.edu }}%
\thanks{This work was supported by the Office of Naval Research Grant N00014-18-1-2830, SUTD-DSO Project Astralis, and Boston University startup funds. We are grateful for their support.}
}

\begin{document}

\maketitle
\thispagestyle{empty}
\pagestyle{empty}

\begin{abstract} 
We apply a novel framework for decomposing and reasoning about free space in an environment to a multi-agent persistent monitoring problem. 
Our decomposition method represents free space as a collection of ellipsoids associated with a weighted connectivity graph. The same ellipsoids used for reasoning about connectivity and distance during high level planning can be used as state constraints in a Model Predictive Control algorithm to enforce collision-free motion. This structure allows for streamlined implementation in distributed multi-agent tasks in 2D and 3D environments. 
We illustrate its effectiveness for a team of tracking agents tasked with monitoring a group of target agents. Our algorithm uses the ellipsoid decomposition as a primitive for the coordination, path planning, and control of the tracking agents.
Simulations with four tracking agents monitoring fifteen dynamic targets in obstacle-rich environments demonstrate the performance of our algorithm.
\end{abstract}

\section{Introduction}

As large-scale nonlinear programming solvers have become more performant, many applications have arisen where autonomous agents use an optimization-based control framework to simultaneously optimize the desired trajectory and necessary control inputs. Recent examples include Model Predictive Control (MPC) for vehicles entering and exiting a platoon \cite{graffione2020non}, UAVs flying in formation subject to radio path loss constraints \cite{grancharova2015uavs}, and videography drones flying subject to visibility and aesthetic constraints \cite{ray2021multi}. However, collision avoidance constraints that prevent crashes with the surrounding environment can present problems when scaling to more general environments. The algorithms often rely on strict assumptions about obstacle structure, do not scale well with environment complexity, or are difficult to implement efficiently for real-time execution. Moreover, it can be difficult to connect the local collision avoidance constraints with global progress of the agent. It is desirable to reason about high level task coordination with the same framework used to eventually enforce collision-free motion.

We present an algorithm that utilizes an ellipsoidal decomposition of the environment to unify reasoning about high level planning and lower level optimization-based control. Our algorithm begins by decomposing the operational area into a set of ellipsoidal, obstacle-free regions. We use a novel method to connect these regions in a graph structure that reflects both connectivity and distance. Agent coordination can be reasoned about in the graph representation, and the same geometric primitives used to decompose the space are used as state constraints in an MPC algorithm. The algorithm scales very well with the complexity of the environment and is directly applicable to both 2D and 3D domains.                                                      

We demonstrate the decomposition framework in a multi-agent persistent patrolling task evocative of automated wild life census, behavior logging, event tracking, or specialized close-up videography. These applications require capabilities for tracking of specific features of the dynamic objects within given constraints. A key requirement of such capabilities is high-level coordination between agents to ensure each agent's actions are useful to the whole. We demonstrate that our decomposition framework is useful in both the task coordination and low level control for a team of $N$ tracking agents patrolling a group of $M$ target agents. Our decomposition-based assignment algorithm performs favorably compared to obstacle-oblivious Voronoi-based assignment methods in two simulation environments.                                                                   

\begin{figure}
    \centering
    \includegraphics[width=\columnwidth]{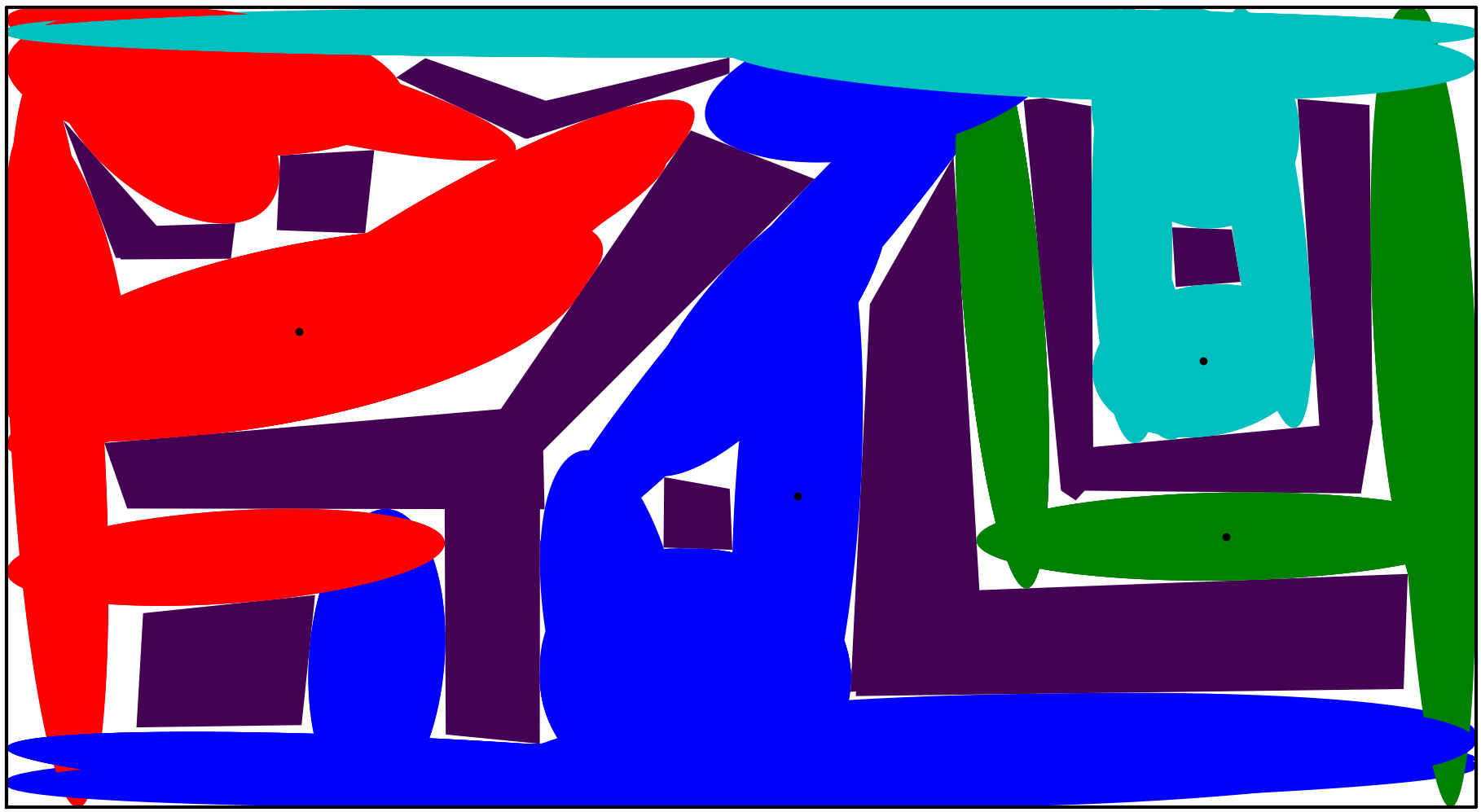}
    \caption{Example ellipsoid decomposition of the ``Jagged'' environment and region assignment for team of tracking agents, with obstacles shown in purple and region assignments shown in red, green, blue, and cyan.}
    \label{fig:ellipse_assignment_custom2}
\end{figure}

The main contributions of the paper are:                                        
\begin{itemize}
    \item A novel method of decomposing and reasoning about free space in obstacle-rich environments that serves as a primitive for task assignment, path planning, and control
    \item An MPC implementation that leverages the free space decomposition generated by higher-level planning
    \item A demonstration of the proposed framework with a complex multi-agent persistent monitoring problem
\end{itemize}       
The rest of the paper is organized as follows: We present related work in Section \ref{sec:related_work}. Section \ref{sec:decomposition_and_representation} presents our central decomposition and graph representation algorithm. We combine the elllipsoidal free-space representation with MPC in Section \ref{sec:mpc}. Finally we demonstrate these techniques in a multi-agent monitoring problem in Section \ref{sec:case_study}.

\section{Related Work}\label{sec:related_work}
The fields of computer graphics, computational geometry, and robotics have extensive literature on polygonal decompositions. Seidel's trapezoidal algorithm \cite{seidel1991simple} is a popular option for 2D decompositions. The Delaunay triangulation \cite{lee1980two} is also often used in 2D and 3D. While both of these algorithms are efficient and provide complete decompositions of a region, the small individual pieces that they generate are a poor match for constraints in optimization-based control. We would instead prefer a smaller number of large regions in the decomposition. Finding the minimum number of convex polyhedra to cover a space is NP-Hard \cite{Culberson1988CoveringPI, abrahamsen2021}, although there has been work in the direction of approximate minimal covers. Unfortunately existing solutions in this direction have various shortcomings (e.g. polynomial but large runtime \cite{eidenbenz2003} or restrictions to 2D \cite{feng1975}) that make them poor options for the 2D and 3D planning and control tasks that we have in mind. However, minimizing the total number of elements in the decomposition is less important than the quality of each. Even if the number of convex elements is not minimal, if each is large then a path between two points may only require traversing a small number of regions. To this end, we build upon the IRIS algorithm \cite{deits-iris} that finds a single large convex region around a seed point.

The ideas in \cite{sarmientoy2005} are the closest to our own. In \cite{sarmientoy2005}, the authors also decompose the free space into a collection of large convex regions to aide path planning. In contrast, our fixed-size representation of each convex region is more appropriate for optimization-based control, and our novel distance metric for building the connectivity graph is more general-purpose than their task-specific graph building.

Our example patrolling task draws inspiration from prior work in multi-agent target-tracking, coverage, persistent monitoring, and pursuer-evader games. One approach to tracking multiple targets with multiple agents is to divide the environment into regions among the trackers, such that each tracker is responsible for their subset of the environment. These coverage control approaches have been demonstrated in non-convex domains \cite{Stergiopoulos2015, Papatheodorou2017, Pierson2017b,Collins2021} or when the total number of targets may be unknown \cite{Zhou2019, Dames2020}. These coverage problems can also include constraints on viewpoints \cite{Petrlik2019}, or incorporate additional camera controls \cite{Hoenig2016, Arslan2018} of the tracking vehicles. In persistent monitoring problems, the tracking agents must design strategies to continuously revisit targets within an environment, such as minimizing the time between observations of a given target \cite{Alamdari2014}. Other work focuses on static targets \cite{Yu2016} or non-convex environments \cite{PalaciosGasos2019}. 

Many of these existing works either simplify dynamic and kinematic constraints in order to provide capture guarantees, or consider the full reachability analysis at great computational cost. We demonstrate that our ellipsoid-based graph representation of the environment can address the target assignment and agent control problems in a unified framework and supports realistic, general dynamics models and realtime control.  In the same vein as the author's prior work on cooperative target tracking with MPC \cite{Pierson2016}, we show the benefits of jointly reasoning about high-level task coordination and low level control. Here, we focus on how our novel decomposition enables the high level task planning to be carried out with the same geometric primitives used to ensure safe motion by the MPC. 

\section{Ellipsoid Decomposition and Representation}\label{sec:decomposition_and_representation}
\subsection{Decomposition}
We compute an approximate decomposition of the environment's obstacle-free space as a union of convex regions using the IRIS library \cite{deits-iris}. Given a seed point, IRIS finds hyperplanes that separate the seed point from all obstacles. The seed point is expanded to the maximum volume ellipsoid enclosed by the hyperplanes, and the process repeats by finding new hyperplanes that separate the ellipsoid from the obstacles. We generate a lattice $\mathcal{L}$ of points with spacing $d$ and begin with no regions in the free-space decomposition $\mathcal{E}$. Then, for each point in $l \in \mathcal{L}$, if it is not contained by any region in $\mathcal{E}$, we use IRIS to add a new obstacle-free region with $l$ as the seed point, detailed in Algorithm \ref{alg:ellipse_decomposition}. 

The regions returned by IRIS can be represented in two ways: the bounding hyperplanes of the polyhedral regions, or the maximum volume ellipsoid constructed within them. Using the region defined by the hyperplanes is a more complete decomposition, but we prefer the ellipsoids for two reasons. First, the description of each region (a centroid and shape matrix) is of a constant size, while the polyhedral regions have varying numbers of sides. Constant-size region description yields a more desirable MPC implementation, as the resulting nonlinear program also has a constant size. Second, ellipsoid-based region descriptions lead to an elegant method of spatially relating free-space regions to each other. We denote ellipsoids as $\mathcal{M}(M, \mathbf{m}) = \{x~|~\norm{x - \mathbf{m}}_M \leq 1\}$, where $\norm{w}_A^2 = w^TAw$.

\begin{algorithm}
\caption{Ellipsoid Decomposition}\label{alg:ellipse_decomposition}
\begin{algorithmic}[1]
\Procedure{EllipsoidDecomposition}{$\mathcal{O}$}
\State $\mathcal{E} = \emptyset$ \Comment{Ellipsoid set initialized empty}
\State $\mathcal{L} \gets $ LatticePoints($d$) \Comment{Generate lattice points}
\State $\mathcal{L} \gets \mathcal{L} \setminus \mathcal{L} \cap \mathcal{O}$ \Comment{Cull points in obstacles}
\While{$|\mathcal{L}| > 0$}
\State $l \gets l \in \mathcal{L}$ \Comment{Choose new lattice point}
\State $\mathcal{L} \gets \mathcal{L} \setminus l$ \Comment{Remove point from set}
\State $\mathcal{E} \gets $ InflateEllipsoid($l, \mathcal{O}$) $\cup~\mathcal{E}$ \Comment{Inflate ellipsoid from lattice seed point}
\For{$p \in \mathcal{L}$}
\If{$p \in \mathcal{E}$}
    \State $\mathcal{L} \gets \mathcal{L} \setminus p$ \Comment{Cull points covered by ellipsoid set}
\EndIf
\EndFor
\EndWhile
\State \textbf{return} $\mathcal{E}$
\EndProcedure
\end{algorithmic}
\end{algorithm}

\begin{figure}
        \centering
        \includegraphics[width=\columnwidth]{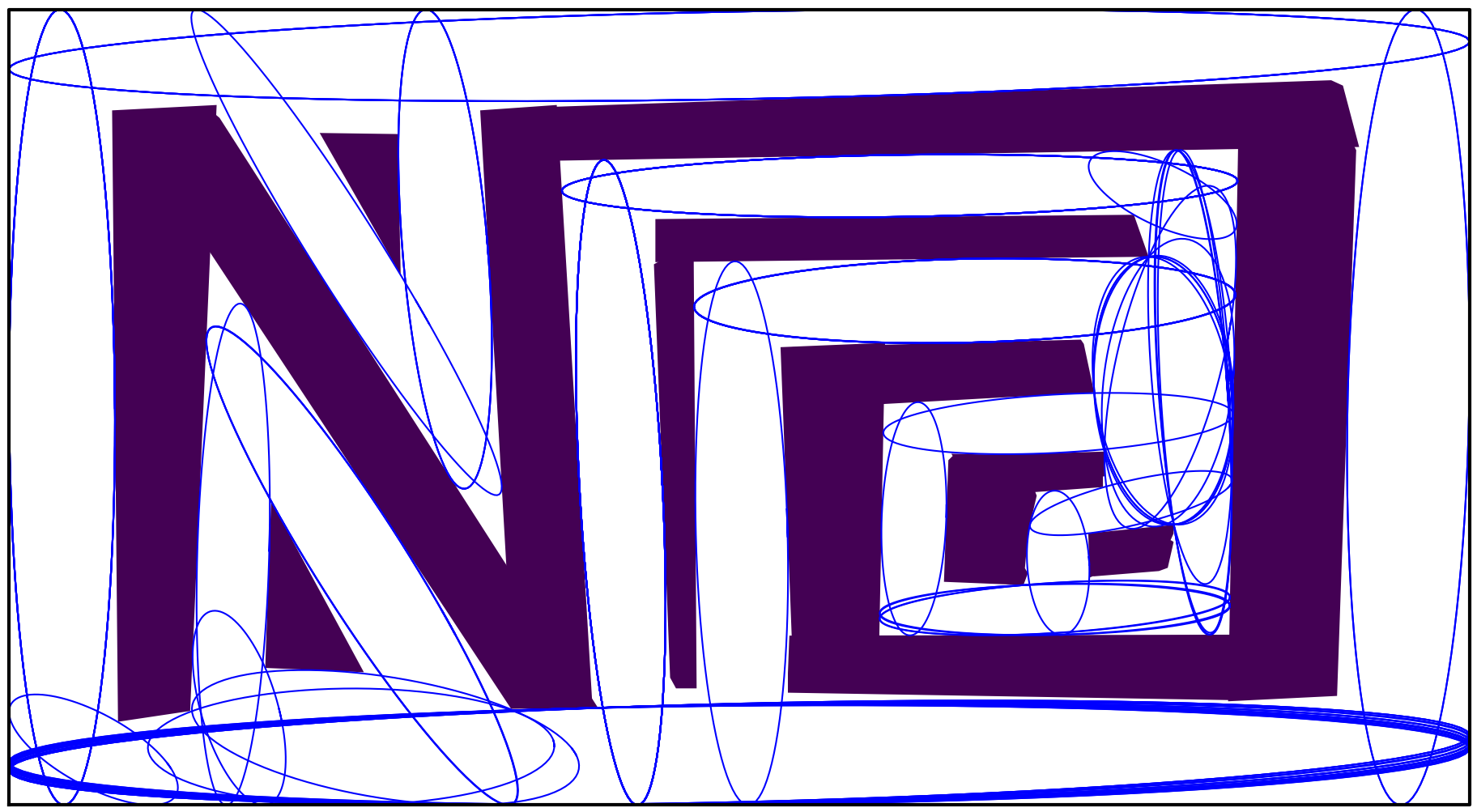}
        \caption{Example ellipsoid decomposition generated by Algorithm \ref{alg:ellipse_decomposition}. Obstacles are represented in purple and the free-space ellipsoids in blue.}
    \label{fig:ellipsoid_decomp_and_assignment}
\end{figure}

\subsection{Graph Construction}\label{sec:graph_construction}
To facilitate using this collection of ellipsoids for spatial planning, we construct a graph $\mathcal{G} = (V, \mathscr{E})$ that represents the connectivity of the ellipsoids. Each ellipsoid $\mathcal{E}_j \in \mathcal{E}$ is represented by a vertex $v_j \in V$, and two vertices share an edge if their associated ellipsoids intersect. We next present an efficient method for checking ellipsoid intersection, which has the benefit of finding a point in the intersection of the ellipsoids if one exists.

For two ellipsoids $\mathcal{A}(A,\mathbf{a})$ and $\mathcal{B}(B, \mathbf{b})$, let 
\[E_\lambda = \lambda A + (1-\lambda)B,~~\mathbf{m}_\lambda = E_\lambda^{-1}(\lambda A\mathbf{a} + (1 - \lambda)B\mathbf{b}),\] for some $\lambda \in [0,1]$. As proven in \cite{igor-2014}, $\mathbf{m}_\lambda$ is in both ellipses simultaneously for some value of $\lambda$ on $[0,1]$ if and only if the two ellipsoids intersect. For $\lambda = 0$ and $\lambda = 1$, $\mathbf{m}_\lambda$ is equal to the origin of ellipse $\mathcal{B}$ and $\mathcal{A}$ respectively. A bisection search on $\lambda$ will either return a point in the intersection of $\mathcal{A}$ and $\mathcal{B}$, or a point $\mathbf{m}_{\tilde\lambda}$ that is not in either ellipsoid--a certificate that the ellipsoids do not intersect.

The correctness of our bisection search relies on $\mathbf{m}_\lambda$ moving monotonically from $\mathcal{B}$ to $\mathcal{A}$. Proposition \ref{prop:monotonic} proves the monotonic behavior for axis-aligned ellipsoids with a diagonal shape matrix. In practice, we observe that monotonicity of $\mathbf{m}_\lambda$ holds for general ellipsoids. 

\begin{prop}\label{prop:monotonic}
Consider two axis-aligned ellipsoids $\mathcal{A}(A, \mathbf{a})$ and $\mathcal{B}(B, \mathbf{b})$ with diagonal shape matrices $A$ and $B$. $\norm{\mathbf{m}_\lambda - \mathbf{b}}_B$ is monotonically increasing for $\lambda \in [0, 1]$.
\end{prop}
\begin{proof}
First note that $\norm{\mathbf{m}_\lambda - \mathbf{b}}$ is translation invariant. Without loss of generality, let $\mathbf{b} = 0$. Making use of the fact that diagonal matrices commute with each other,
\begin{equation*}
\begin{split}
    \norm{\mathbf{m}_\lambda}_B^2 &= \\
    \mathbf{a}^T&A^T(\lambda A + (1-\lambda)B)^{-T}B(\lambda A + (1-\lambda)B)^{-1}A\mathbf{a}\\
    &=\sum_i \mathbf{a}_i^2B_{ii}\left(\frac{\lambda A_{ii}}{\lambda A_{ii} + (1-\lambda)B_{ii}}\right)^2\\
    &= \sum_i \mathbf{a}_i^2B_{ii}\left(\frac{\lambda}{\lambda + (1-\lambda)\frac{B_{ii}}{A_{ii}}}\right)^2.
\end{split}
\end{equation*}
As $A$ and $B$ are positive definite, $\frac{B_{ii}}{A_{ii}} \in (0, \infty)$. Additionally,
\begin{equation*}
\frac{\lambda}{\lambda + (1-\lambda)\frac{B_{ii}}{A_{ii}}},
\end{equation*}
is monotonically increasing in $\lambda \geq 0$, for any $\frac{B_{ii}}{A_{ii}} \in (0, \infty)$. As $\mathbf{a}_i^2B_{ii}$ is also positive, $\norm{\mathbf{m}_\lambda}_B^2$ is monotonically increasing for $\lambda \in [0,1]$. Squaring a positive function preserves monotonicity, so $\norm{\mathbf{m}_\lambda - \mathbf{b}}_B$ is monotonically increasing for $\lambda \in [0, 1]$, thus completing our proof.
\end{proof}

The length of the path $\mathbf{m}_\lambda$ is an upper bound on the collision-free distance between the centroids of two ellipsoids. We define a distance heuristic $\hat{d}$ as
\begin{equation*}
    \hat{d}(\mathcal{E}_1, \mathcal{E}_2) = \int_0^1 \norm{\frac{d\mathbf{m}_\lambda}{d\lambda}} d\lambda,
\end{equation*}
which we compute numerically. The weighted edge set $\mathscr{E}$ of the ellipsoid connectivity graph is then defined by
\begin{equation*}
   \mathscr{E} = \left\{ (i, j, \hat{d}(\mathcal{E}_i, \mathcal{E}_j)) ~|~ \mathcal{E}_i \cap \mathcal{E}_j \neq \emptyset \right\},
\end{equation*}
for $\mathcal{E}_i, \mathcal{E}_j \in \mathcal{E}$, where $e = (i, j, w)$ denotes an edge from vertex $i$ to vertex $j$ with weight $w$. The resulting edge weights may violate the triangle inequality. To make the graph metric and provide a tighter bound on the shortest distance between ellipsoid centers, the weight of each edge can be updated to the weight of the lowest-weight path between its two endpoints in the original graph.

The graphical interpretation of environment free space, $\mathcal{G}$, and the associated obstacle-free ellipsoids can now be used for task assignment and path planning, while the associated ellipsoids are useful for collision-free motion planning.

\subsection{Path Planning}\label{sec:path_planning}
To move around the environment while avoiding obstacles, a mobile agent must plan a path in the ellipsoidal decomposition graph $\mathcal{G}$ from its current location $\mathbf{x}_{start}$ to some $\mathbf{x}_{goal}$. To find such a path, we augment $\mathcal{G}$ with an additional vertex for the start and goal positions. Each of the two additional vertices shares an edge with vertices corresponding to the ellipsoids containing that point. The augmented vertex set is $V' = V \cup \{v_{start}, v_{goal}\}$, and the augmented edge set is
\begin{align*}
    \mathscr{E}' = \mathscr{E} &\cup \{(v_{start}, j, \norm{\mathbf{x}_{start} - \mathbf{e}_j}) ~\vert~ \norm{\mathbf{x}_{start} - \mathbf{e}_j}_{E_j} \leq 1\} \\
                 &\cup \{(v_{goal}, j, \norm{\mathbf{x}_{goal} - \mathbf{e}_j})
    ~|~\norm{\mathbf{x}_{goal} - \mathbf{e}_j}_{E_j} \leq 1 \},
\end{align*}
for $\mathcal{E}_j(E_j, \mathbf{e}_j) \in \mathcal{E}$.

The shortest path from $v_{start}$ to $v_{goal}$ in the augmented graph $\mathcal{G}' = (V', \mathscr{E}')$, which we denote $\mathcal{E}^*$, corresponds to an upper bound on the minimum collision-free distance from $\mathbf{x}_{start}$ to $\mathbf{x}_{goal}$. If the start and goal points are within the same ellipsoid, this pathfinding method may return a longer ellipsoid sequence than the single shared ellipsoid, depending on their position relative to the current ellipsoid's center. We explicitly check for this case, and if the start and end points share an ellipsoid we set $\mathcal{E}^*$ to be the current ellipsoid. In the patrolling case study, we pre-compute the shortest path between all pairs of nodes with the Floyd-Warshall algorithm \cite[Chapter~5]{clrs}.

We use $\mathcal{E}^*$ to construct a sequence of waypoints between $\mathbf{x}_{start}$ and $\mathbf{x}_{goal}$. A point in the intersection of each successive pairs of ellipsoids is added to the sequence, computed as described in Section \ref{sec:graph_construction}. 

\section{Ellipsoid Decomposition with MPC}\label{sec:mpc}

Given a sequence of waypoints from the graph-based path planning algorithm, we use a flexible Model Predictive Control (MPC) formulation to solve for the agent control inputs. Progress toward the next waypoint is achieved with a distance-based stage cost over the planning horizon, and a dynamically updated ``cost switching'' index ensures further progress toward the following waypoint. A similar updating index constrains each point on the planning horizon to be inside one of two free-space ellipsoids.

\subsection{Cost Switching}
To make progress toward the final goal position, we expect the agent to move toward the next waypoint in the sequence generated by the path planning algorithm. In practice, this intermediate waypoint may be within the planning horizon. It would be undesirable for the waypoint to be fixed for the entire planning horizon, which would lead to a trajectory that slows down or stops at the waypoint. Instead, part of the trajectory has a cost function related to the closest waypoint, while the rest is incentivized to reach the following waypoint.

We denote the cost functions related to the first and second waypoints as $l^1(\mathbf{x})$ and $l^2(\mathbf{x})$. In the simple case where the agent's only objective is moving between a fixed start and end point, then $l^1$ and $l^2$ is the distance between the target and corresponding waypoint. Our case study in Section \ref{sec:case_study} shows an example of a more complicated cost function. We define a \textit{cost switching index} $h_{t}$, which determines when the MPC switches from the primary cost function $l^1$ to the secondary cost function $l^2$. Initially, $h_t$ is set to $N+1$, one larger than the MPC horizon. Each time a solution is returned from the MPC, the predicted trajectory for the agent is compared to the intermediate waypoint location. $h_t$ is set to the first index where the tracking trajectory enters within some radius of the waypoint. After this point, it spends the rest of the horizon optimizing for the secondary goal. As the agent approaches the first waypoint, $h_t$ decreases and more of the trajectory is pulled toward the second waypoint. This effect is illustrated in Figure \ref{fig:switching_index_example}.

\begin{figure}
    \centering
    \includegraphics[width=\columnwidth]{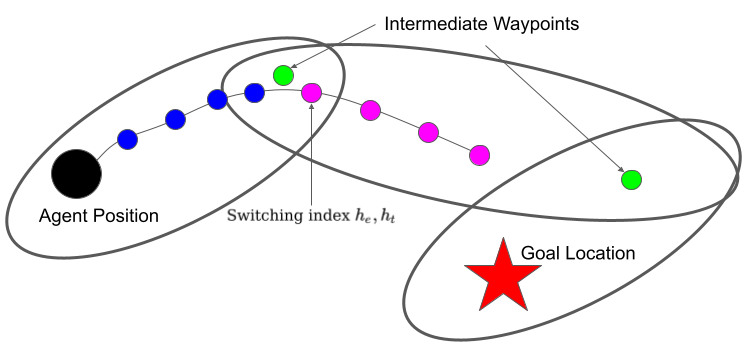}
    \caption{Example of switching indices $h_t$ and $h_e$. The first half of the trajectory (in blue) optimizes cost $l^1(x)$ based on distance to intermediate waypoint. The second half of the trajectory (in magenta) optimizes cost $l^2(x)$ toward the following waypoint. The switching indices are updated between MPC iterations depending on the predicted trajectory. In general, $h_e$ and $h_t$ need not be the same. }
    \label{fig:switching_index_example}
\end{figure}

\subsection{Constraint Switching}
Staying within the union of obstacle-free ellipsoids $\mathcal{E}$ allows an agent to avoid collisions with the environment. However, the union of ellipsoids is nonconvex, and imposing the constraint directly onto the MPC increases solution times and leads to poor local optima. Previous works \cite{deits-uav-mip}, \cite{landry2016aggressive} address this issue by formulating the constraints as a mixed-integer convex nonlinear program, where exactly one convex region is the active constraint at each timestep. Unfortunately, requiring the solver to support integer constraints removes the possibility of using faster continuous nonlinear solvers. Instead, we employ a \textit{constraint switching index} $h_e$ analogous to the cost switching index that controls which free-space ellipsoid constrains each stage of the MPC.

The ellipsoids that we choose come from the path planning solution $\mathcal{E}^*$. The first ellipsoid contains the agent and first waypoint. The second ellipsoid contains the first and second waypoints. We denote the first and second collision avoidance ellipsoids as $\mathcal{M}_1(M_1, \mathbf{m}_1)$ and $\mathcal{M}_2(M_2, \mathbf{m}_2)$. The first ellipsoid constraint is active until step $h_e$ of the MPC horizon. The second ellipsoid constraint is active from step $h_e$ onward. This results in a pair of quadratic constraints for each agent, $\norm{\mathbf{x}_k - \mathbf{m}_i}_{M_i}^2 \leq d_k^i,~i \in \{0, 1\}$,
defined for each step $k$ of the MPC horizon.

Setting $d_k^i$ to 1 constrains $\mathbf{x}$ to be within ellipsoid $\mathcal{M}_i$ at time $k$. Setting $d_k^i$ to $\infty$ turns off the constraint at time $k$. We set $d_k^0$ to $\infty$ if $k \geq h_e$, and 1 otherwise. Similarly, $d_k^1$ is equal to $\infty$ if $k < h_e$, and 1 otherwise. Thus, in the interval $[0, h_e)$, only the $\mathcal{M}_1$ ellipsoid constraint must be satisfied, and vice-versa for $[h_e, N-1]$. Within each of these intervals, the feasible positions for $\mathbf{x}$ are convex, which yields more reliable MPC performance in practice. We update $h_e$ in a similar manner to $h_t$, based on what portion of the previous MPC solution was contained in the second ellipsoid.

\subsection{Full Formulation}

The full MPC formulation can be written as a constrained nonlinear optimization problem:

\begin{subequations}\label{eq:mttmpc}
    \begin{alignat}{2}
        \argmin_{\mathbf{w}_{1:N}, \mathbf{u}_{1:N}}  ~~
        & \sum_{k=1}^{N} l_k^1(\mathbf{x}_k)\mathds{1}(k < h_t) + l_k^2(\mathbf{x}_k)\mathds{1}(k \geq h_t) + q(\mathbf{u}_k)\\
        \text{s.t.} ~~          & \mathbf{w}_1 = \mathbf{w}(0), \\
        & \dot{\mathbf{w}}(t_{c,n}) = f(\mathbf{w}(t_{c,n}), \mathbf{u}(t_{c,n})),\label{eq:collocation}\\
        & \norm{\mathbf{x}_k - \mathbf{m}_i}_{M_i}^2 \leq d^i_k, ~i\in \{0,1\}\\
        & \mathbf{w} \in \mathcal{W}, \mathbf{u} \in \mathcal{U},~~ \forall k\in \{1,\dots,N\}.
    \end{alignat}
\end{subequations}
In our implementation, the system dynamics $\dot{\mathbf{w}} = f(\mathbf{w}, \mathbf{u})$ are enforced with a third order collocation method \eqref{eq:collocation}, as discussed in \cite{underactuated}. $t_{c,n}$ denotes the time corresponding to the $n^{th}$ collocation point, and the resulting constraint can be represented directly in terms of $\mathbf{w}_{1:N}$ and $\mathbf{u}_{1:N}$. A control effort regularization function $q$ is added to the cost function to encourage smoother control inputs. We let $q(\mathbf{u})$ be proportional to $\norm{\mathbf{u}}^2$. The first control input $\mathbf{u}_1$ is applied to the system, and then the optimization re-solved in a receding horizon manner. 

Note that the number of constraints necessary to enforce collision-free motion is constant and independent of the environment's complexity. While a polytope representation of free space results in a varying number of halfplane constraints, here we always have two constraints based two ellipsoids. This greatly aids MPC implementations where the solver is pre-compiled or reused between iterations.

\section{Persistent Monitoring Case Study}\label{sec:case_study}
We consider a multi-agent patrolling problem to demonstrate how our ellipsoidal decomposition algorithm provides a useful tool for jointly reasoning about planning and control. A team of tracking agents must periodically visit a group of moving target agents from a relative viewpoint. We demonstrate a task assignment algorithm that uses the decomposition graph $\mathcal{G}$ to reason about task assignments for the agents, and then uses the ellipsoids associated with paths in $\mathcal{G}$ to impose collision avoidance constraints in a viewpoint-aware MPC. We evaluate the proposed algorithm in a 2D simulation in two different environments that feature non-convex obstacles of varying shape and size, shown in Figure \ref{fig:ellipse_assignment_custom2} (Jagged Environment) and Figure \ref{fig:ellipsoid_decomp_and_assignment} (Spiral Environment). 

\subsection{Problem Formulation}
Here, a team of $N$ trackers seeks to observe $M$ targets, with target-centric viewpoints specified as objectives. The team of tracking agents must be controlled to minimize the time between observations of each target. For each target $m \in \{1,...,M\}$ we denote the desired viewing direction $\boldsymbol\eta$, viewpoint tolerance $\theta$, and maximum distance $r_{max}$. For a tracking agent to successfully observe target $m$, it must enter a circular arc relative to $m$ with medial axis $\boldsymbol\eta$, angle $\theta$, and radius $r_{max}$. 

Tracking agents move subject to second order unicycle dynamics, with the state consisting of position, heading angle, linear velocity, and angular velocity: $\mathbf{w} = \begin{bmatrix}x;~y;~\theta;~v;~\omega \end{bmatrix}$,
and the control input consisting of longitudinal and angular acceleration: $\mathbf{u} = [a;~ \alpha]$
The dynamics evolve according to $\dot{\mathbf{w}} = [v\cos(\theta);~v\sin(\theta);~\omega;~a;~\alpha]$.

The tracking agents are constrained to have a velocity in the range $[0, 3] m/s$, with acceleration constrained to be in $[-1, 1] m/s^2$. Angular velocity is constrained to $[-3, 3]$ rad/$s$, with angular acceleration constrained to $[-2, 2]$ rad/$s^2$. The acceleration constraints and second order dynamics make this a nontrivial control problem, even before considering the viewpoint cost function.

The target agents are given random obstacle-free destination points in the environment and move according to a first order unicycle model with a constant velocity of $0.3 m/s$ (angular velocity is controlled directly). They plan intermediate waypoints with the method outlined in Section \ref{sec:path_planning}, and heading is driven with a Proportional Derivative (PD) feedback controller. By virtue of the waypoint choice, the target agents are usually in obstacle-free space, although they may occasionally drive through obstacles.


\subsection{Task Assignment}
Before task execution, the approximate ellipsoidal decomposition of free space $\mathcal{G}$ is constructed. At runtime, each tracking agent is assigned a subset of the targets it is responsible for observing. This subset updates according to changes in distribution of targets. Each tracking agent plans an order for visiting each of its assigned targets. The assignment order is used to generate a collision-free reference path for the tracker, which is followed by an MPC.

We examine two complementary methods for the assignment step. The first method dynamically updates each agent's region of responsibility online. Let $\hat{D}(\mathbf{x}_i, \mathbf{x}_j)$ return the shortest distance based on the approximation from $\mathcal{G}'$. Each target is assigned to the closest tracker, as estimated by $\hat{D}$. 
The second method finds a static partition of the environment and assigns each tracker agent to a fixed region. We partition the set of free-space ellipsoids among the tracking agents with a greedy k-center approximation problem. For a metric graph $\mathcal{G} = (\mathcal{V}, \mathscr{E})$, the k-center problem is to find a subset of vertices $\mathscr{V}\in\mathcal{V}$ such that the maximum distance between any vertex in $\mathcal{G}$ and a vertex in $\mathscr{V}$ is minimized. The problem is NP-Complete, but there is a simple heuristic that guarantees at most twice the optimal distance: with $\mathscr{V}$ initialized with a single random node, add the node in $\mathcal{V}$ that is furthest from any node in $\mathscr{V}$ and repeat $k - 1$ times \cite{kcenter_2opt}. Such a decomposition is illustrated in Figure \ref{fig:ellipse_assignment_custom2}. Each tracking agent is assigned responsibility for the targets currently in its designated partition. While the regions have some overlap, they generally ensure that the tracking agents are distributed around the environment. The two methods of dividing targets can be used interchangeably, depending on whether it is desirable to have each agent confined to a known region.

Once the targets have been divided among the tracking agents, a greedy heuristic $h$ is used to guide the tracker's visitation order of its assigned targets. The visitation order heuristic is a weighted combination of each target's distance and the amount of time since it has been surveyed. Formally,

\begin{equation*}
    h^i = \hat{D}(\mathbf{x}_{tracker}, \mathbf{x}_{target}^i) - w_{staleness} (t_{now} - t_{seen}^i).
\end{equation*}

The heuristic balances opportunism in viewing targets that are convenient with an incentive to seek out targets that have not been seen in a while. The targets are pursued in ascending order of $h_i$, although as the ordering is updated at every timestep only the first two targets affect the control policy. 


\subsection{MPC Viewpoint Cost}
We define a viewpoint cost function that encodes the desire to view the target from a specific angle and distance. Consider a target at position $\mathbf{x}_{target}$ with a desired viewing direction $\boldsymbol\eta$. In order to get the desired observation of the target, the tracking agent must enter into an arc with opening angle $\theta$ around $\boldsymbol\eta$, and within distance $r$. Let $\mathbf{\tilde{x}}$ denote the vector from $\mathbf{x}_{target}$ to $\mathbf{x}_{tracker}$. The viewpoint cost function $l$ is composed of a term that penalizes the deviation of $\mathbf{\tilde{x}}$ from the direction of $\boldsymbol\eta$ and the length of $\mathbf{\tilde{x}}$ from the desired viewing distance. The viewpoint cost function is defined as
\begin{equation}\label{eq:viewpoint_cost}
    l(x) = \frac{\mathbf{\tilde{x}}^T}{\norm{{\mathbf{\tilde{x}}}}} \boldsymbol{\eta} + (\norm{\mathbf{\tilde{x}}}^2 - r^2)^2.
\end{equation}

The position and direction of interest of the target are predicted over the horizon of the MPC planning. As a result, $l(\mathbf{x})$ is implicitly a function of time from the time-dependence of $\mathbf{\tilde{x}}$ and $\mathbf{\eta}$. We make this explicit by notating $l_i(\mathbf{x})$ as the cost function relative to the target at time $i$. Our implementation predicts the target's future motion assuming no control inputs are applied to the target over the planning horizon. When a tracker agent's current position relative to its current target enters the desired viewing cone, we consider the target to have been visited. When the tracking and target agent share an ellipsoid or are in adjacent ellipsoids, the cost function $l$ is used as cost $l^1$ or $l^2$ in \eqref{eq:mttmpc}. In this case, the cost switching index $h_t$ can be updated based on when the predicted trajectory enters the target's viewing cone.

\subsection{Simulation}
The viewpoint-optimizing cost function \eqref{eq:viewpoint_cost} for each agent is optimized subject to the agent dynamics by implementing \eqref{eq:mttmpc} in Casadi \cite{Andersson2019} and solving with IPOPT \cite{ipopt}. The MPC considers a five second horizon, comprised of 40 equally-spaced steps. It takes an average of 50ms per MPC solution on an Intel Core i7-9750H CPU @ 2.60GHz. The MPC is given its previous solution as an initial guess. The computation for decomposition and graph construction is dominated by the ellipse inflation at each seed point. The ellipse inflation takes under 6 seconds in both test environments. 

Still frames from our video illustrating the tracking task within Environment 2 are shown in Figure \ref{fig:city_evolution}.  We simulate both the static and dynamic ellipsoid region assignment algorithms. We compare these algorithms to a static and dynamic obstacle-oblivious assignment methods. The first comparison is to a static, obstacle-oblivious Voronoi partition of the space, with the same seed points used as the k-centers defining the ellipsoid partition. The second comparison is a dynamic Voronoi assignment algorithm, where targets are assigned to the closest tracker at every iteration without consideration of obstacles. All four algorithms use the same path planning and MPC implementations. Figure \ref{fig:mean_times_1} presents the mean time between visits for targets, and Figure \ref{fig:max_times_1} presents the maximum time between visits for the targets over 51 runs. Each run randomizes the initial positions of the target agents and simulates 625 seconds of the tracking task.
The static region assignments are slightly different between runs due to the stochastic nature of the k-centers approximation algorithm.

In both domains, the algorithms with dynamically allocated regions of responsibility outperform the static region assignments. This is unsurprising, as dynamic partitioning enables tracking agents to directly respond to shifting distributions of the target agents. In both environments, the dynamic ellipsoid decomposition method outperforms the dynamic Voronoi method, for both metrics of mean visit delay and maximum visit delay. The static ellipsoid decomposition slightly underperforms the static Voronoi method in the Spiral environment and outperforms it in the Jagged environment.
The static ellipsoid decomposition compares most favorably when there are larger free-space regions thinly separated by obstacles, which is better reflected in the Jagged environment. Overall, these simulations demonstrate the efficacy of our ellipsoid decomposition and path planning to perform complex multi-agent tasks.

\begin{figure}
    \centering
    \includegraphics[width=.8\columnwidth]{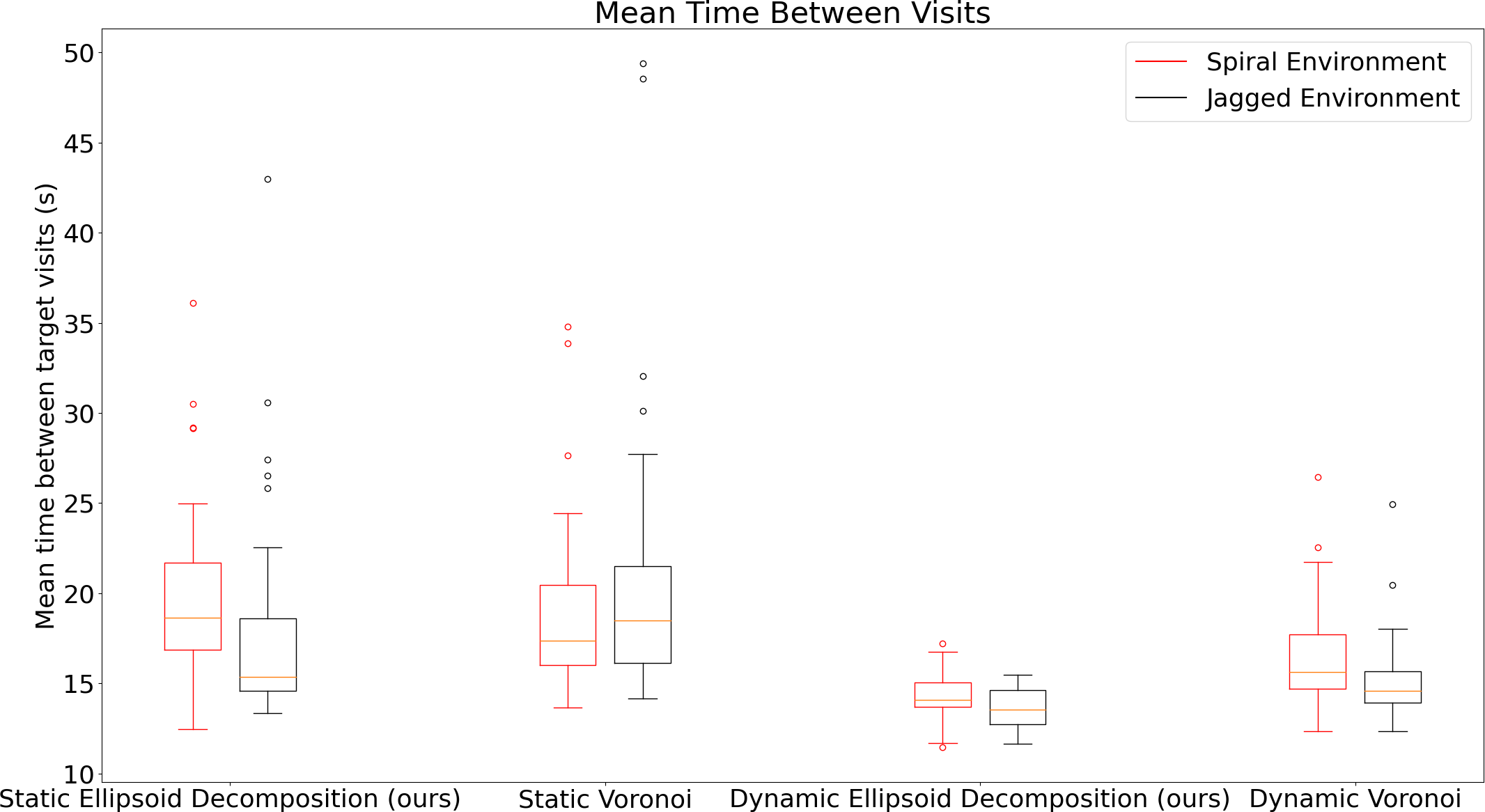}
    \caption{Mean time between target visitation for the four algorithms.}
    \label{fig:mean_times_1}
\end{figure}
\begin{figure}
    \centering
    \includegraphics[width=.8\columnwidth]{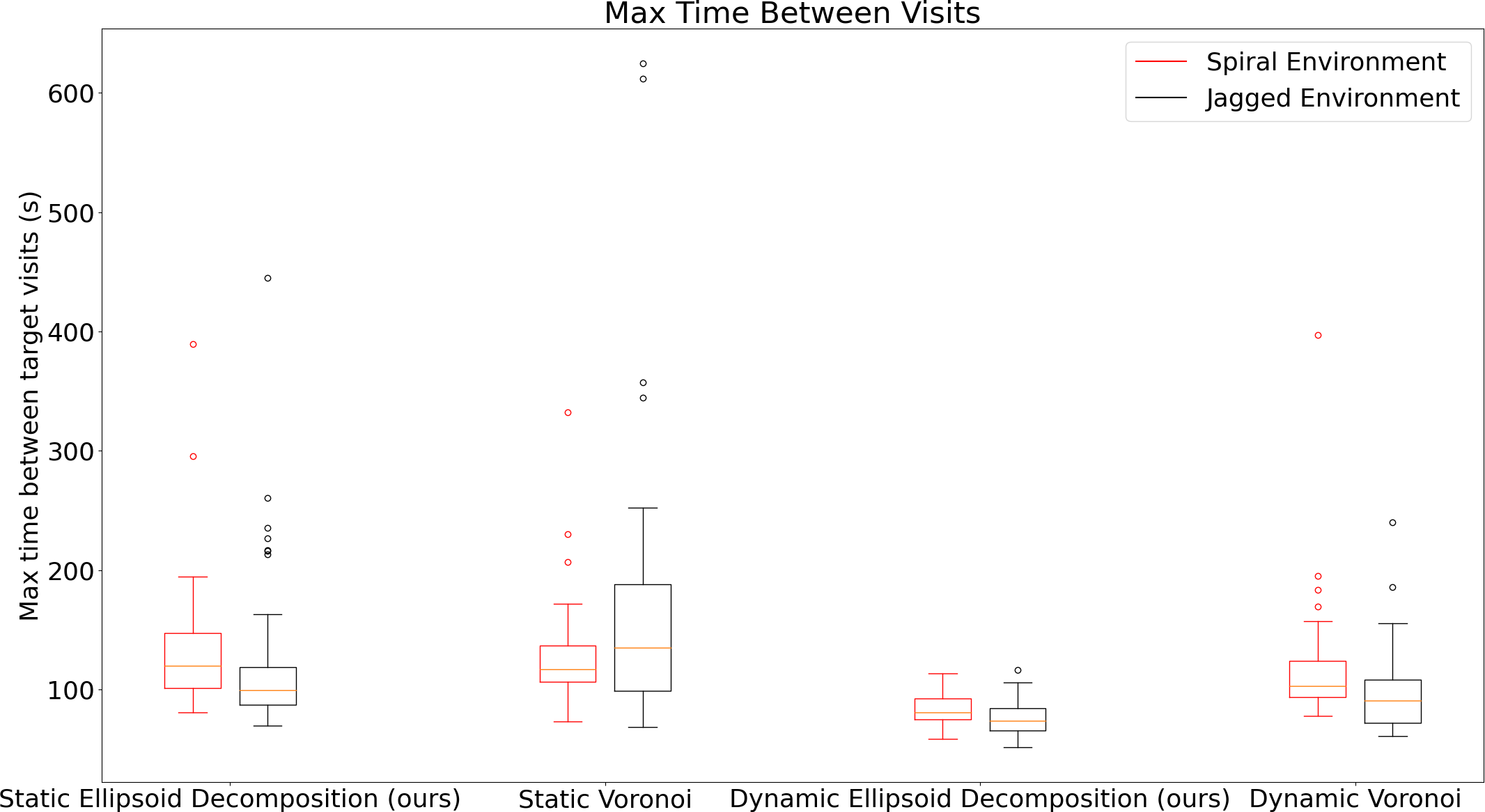}
    \caption{Max time between target visitation for the four algorithms.}
    \label{fig:max_times_1}
\end{figure}

\begin{figure}[h!]
    \begin{subfigure}{0.49\columnwidth}
        \centering
        \includegraphics[width=\textwidth]{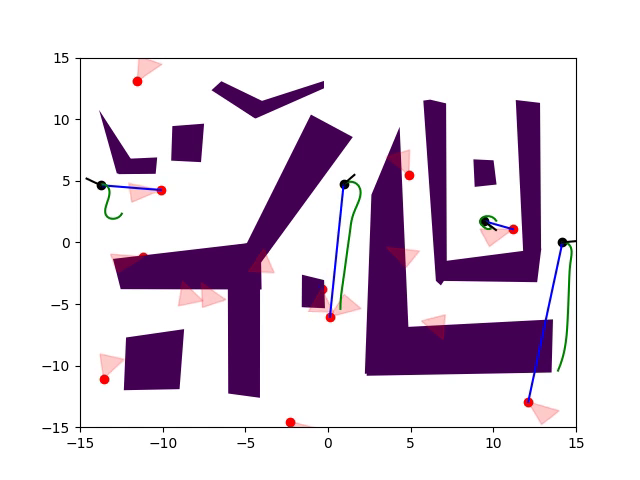}
        \caption{$t = 0.25$s}
    \end{subfigure} 
    \begin{subfigure}{0.49\columnwidth}
        \centering
        \includegraphics[width=\textwidth]{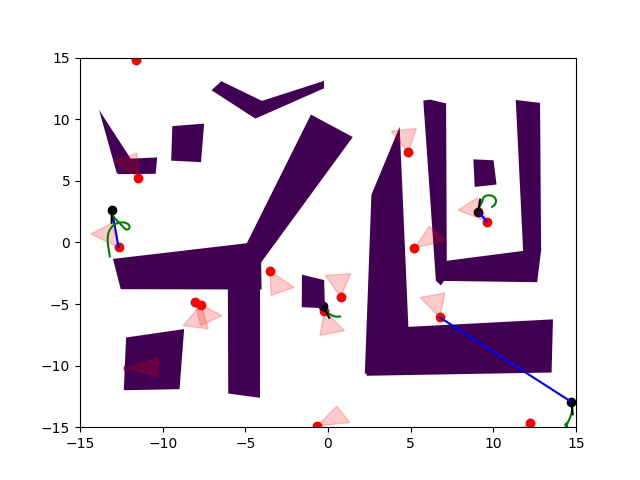}
        \caption{$t = 6.25$s}
    \end{subfigure} 
    
    \begin{subfigure}{0.49\columnwidth}
        \centering
        \includegraphics[width=\textwidth]{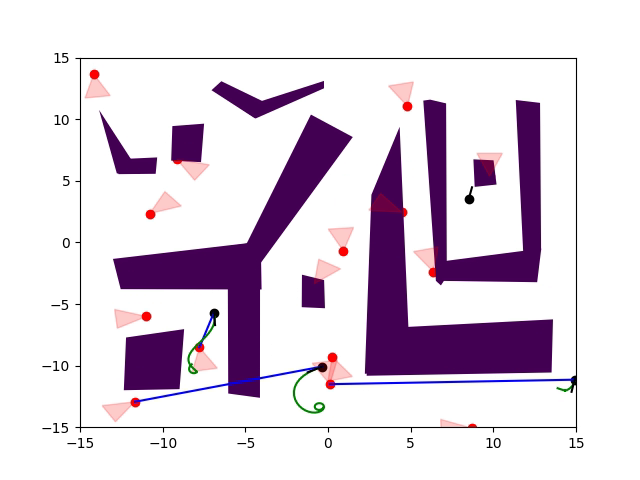}
        \caption{$t = 18.75$s}
    \end{subfigure} 
    \begin{subfigure}{0.49\columnwidth}
        \centering
        \includegraphics[width=\textwidth]{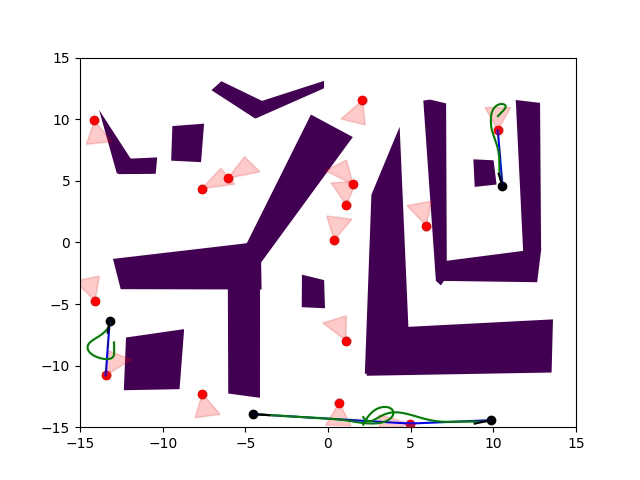}
        \caption{$t = 31.25$s}
    \end{subfigure} 
    
 \caption{Evolution of a tracking simulation over time. Four tracking agents are pictured as black dots. Purple blocks represent obstacles. Red dots represent targets of interest, and red triangles are relative viewpoint from which they must be viewed by a tracker. Green lines are the tracking agents' MPC plan at the current iteration. The dark blue line connects each tracker with its currently-assigned target.} 
\label{fig:city_evolution}
\end{figure}

\section{Conclusions and Future Work}

This paper has demonstrated a novel method for decomposing and reasoning about spatial structure in obstacle-rich 2D and 3D environments. The obstacle aware ellipsoid-based graph decomposition enables high-level coordination between agents, path planning, and collision-free motion planning in a unified representation. We have demonstrated the utility of this framework with a multi-agent persistent monitoring task where an MPC algorithm leverages the decomposition to monitor a team of target agents from desired viewpoints. We hope to extend the applications of this ellipsoid framework to other multi-agent tasks as a coordination and planning primitive.

\IEEEtriggeratref{18}
\bibliographystyle{IEEEtran}
\bibliography{IEEEabrv,references}

\end{document}